\documentclass{article}
\usepackage{iclr2027_conference,times}
\usepackage[T1]{fontenc}
\usepackage{amsmath,amssymb,amsthm}
\usepackage{booktabs,graphicx,float,microtype,tabularx,enumitem}
\usepackage{placeins,needspace}
\usepackage{hyperref}
\usepackage{url}

\usepackage{subcaption} 
\usepackage{algorithm,algpseudocode}

\newtheorem{proposition}{Proposition}

\title{Continual Data Unlearning in Diffusion Models via Transition-Based Regularization}
\author{
Sunbeom Jeong$^{1}$\thanks{These authors contributed equally to this work.} \quad
Sehwan Kim$^{1*}$ \quad
Sangwoo Hong$^{2}$ \quad
Jungwoo Lee$^{1}$ \\
$^{1}$Seoul National University \qquad
$^{2}$Konkuk University
}

\iclrarxivcopy

\begin{document}
\maketitle
\begin{abstract}
Data unlearning in diffusion models aims to remove the influence of specific training examples without suppressing the broader concepts they represent.
However, when deletion requests arrive sequentially, updates for new requests can degrade generative utility and undermine earlier deletions.
We propose a continual data unlearning framework that uses completed deletion transitions as directional references to regularize future updates.
For each request, we record changes in denoiser responses on the same fixed noisy inputs before and after unlearning.
Rather than matching full post-deletion responses, we apply a one-sided penalty that discourages reversal along the recorded directions relative to the post-deletion references, while leaving orthogonal response changes and progress beyond these references unpenalized.
To keep storage independent of the number of requests, we maintain a fixed-capacity bank of representative transition records. Records are selected based on the local sensitivity of progress along their recorded directions to parameter updates, allowing them to be retained even when their penalties are inactive.
Empirical evaluations show that the proposed framework achieves a better balance between deletion persistence and generative utility than existing unlearning baselines as requests accumulate, using only a small transition memory.
\end{abstract}


\section{Introduction}
\label{sec:introduction}

Diffusion models have demonstrated remarkable capabilities in generating high-quality images~\citep{rombach2022high, ramesh2022hierarchical}.
These capabilities are enabled by training on massive image--text datasets collected from the internet~\citep{schuhmann2022laion}.
However, the inclusion of copyrighted material and personal data in these datasets raises a range of ethical and legal concerns~\citep{carlini2023extracting, dubinski2025cdi}.
For instance, diffusion models can memorize and reproduce individual training samples, potentially exposing privacy-sensitive information or replicating copyrighted content~\citep{carlini2023extracting}.
Such concerns motivate mechanisms for addressing training-data deletion requests.
This need is further underscored by privacy laws such as the General Data Protection Regulation (GDPR) and the California Consumer Privacy Act (CCPA), which recognize individuals' rights to request the deletion of their personal data under applicable conditions.
A straightforward approach is to exclude request samples from the training set and retrain the model from scratch.
However, retraining large-scale diffusion models for every deletion request is computationally prohibitive~\citep{alberti2025data}.
This motivates \emph{machine unlearning}, which aims to efficiently remove the influence of these samples without full retraining while preserving model utility~\citep{bourtoule2021machine,alberti2025data}.

Research on unlearning in diffusion models has explored two distinct objectives: \emph{concept unlearning} and \emph{data unlearning}.
Concept unlearning aims to suppress the model's ability to generate particular semantic concepts, such as objects, identities, or artistic styles, as explored by methods such as ESD and Forget-Me-Not~\citep{gandikota2023erasing, zhang2024forget}.
In contrast, data unlearning seeks to remove the influence of specific training examples rather than the broader concepts they represent~\citep{alberti2025data, shi2026retrack}.
Given a training dataset $X$ and a subset $A \subset X$ requested for deletion, the model should ideally behave as though it had been trained from scratch on $X \setminus A$.
For example, unlearning a particular photograph of a dog should not eliminate the model's ability to generate dogs.
The model should still produce high-quality, text-aligned images for prompts associated with deleted samples, provided that the relevant concepts remain supported by the retained data.
Concept-level suppression may therefore restrict generative capabilities beyond the scope of an individual deletion request.
Recent methods directly pursue this datapoint-level objective: SISS introduces an importance-sampling-based unlearning objective~\citep{alberti2025data}, while ReTrack redirects denoising trajectories toward neighboring examples in the retained data~\citep{shi2026retrack}.

\begin{figure}[t]
    \centering
    \includegraphics[width=\linewidth]{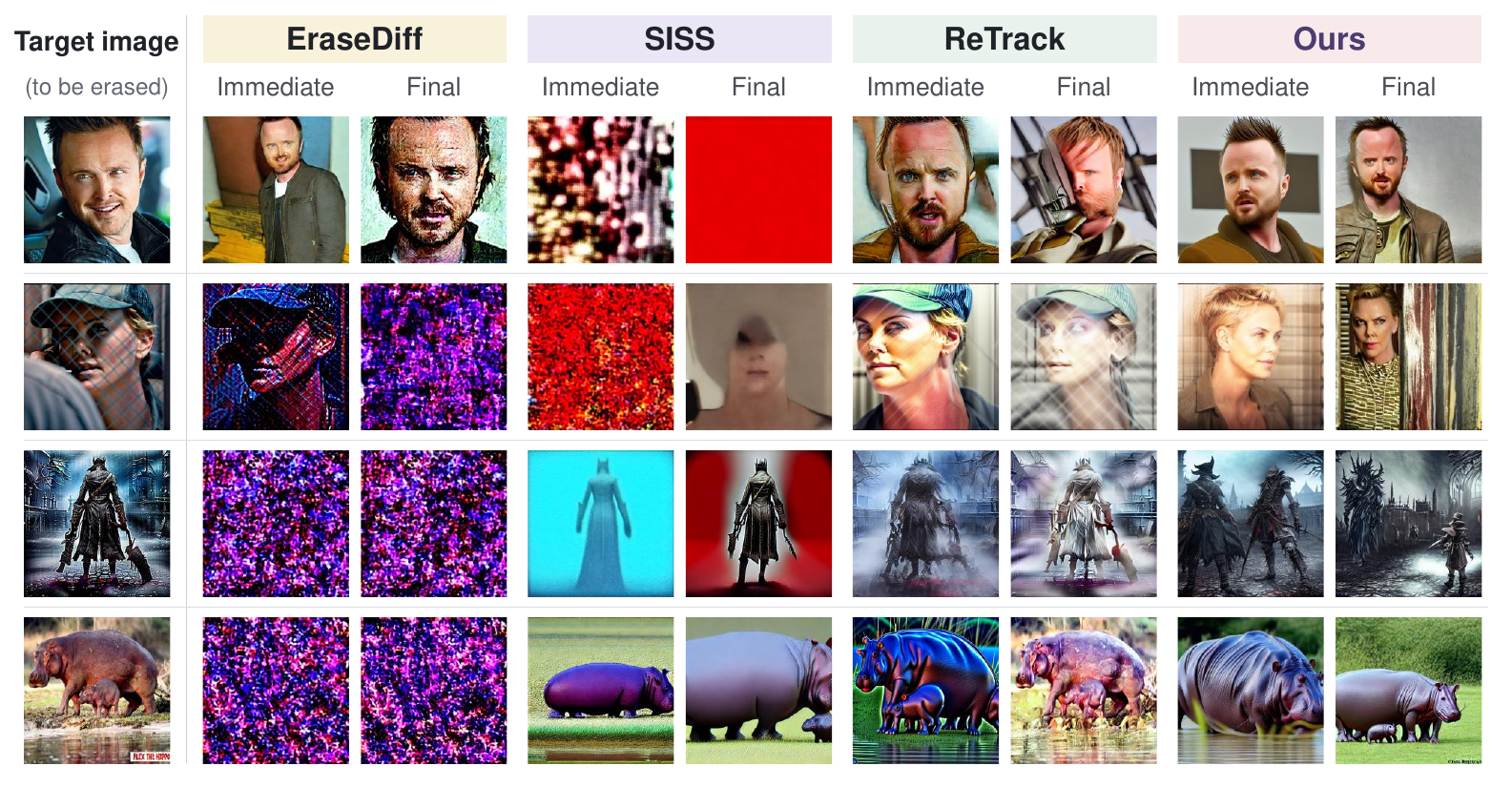}
    \caption{Qualitative comparison of continual data unlearning.
    Each row shows a target image to be erased and samples generated using its associated prompt.
    For each method, “Immediate” denotes generation immediately after deleting the target, while “Final” denotes generation after completing the subsequent deletion requests.
    The prompts used for the generations are provided in Appendix~\ref{app:fig_1_prompt}.
}
    \label{fig:fig1}
\end{figure}

In practical deployments, however, deletion requests may arrive sequentially rather than as a single, predetermined batch, necessitating \emph{continual data unlearning}.
Recent work has explored continual concept unlearning~\citep{lee2026continual,lacu}. 
For data unlearning, however, applying existing methods~\citep{alberti2025data,shi2026retrack} sequentially raises two key challenges.
First, repeated updates may degrade image quality or text--image alignment, including for prompts associated with deleted samples whose underlying concepts remain supported by retained data.
Failure to reproduce deleted samples does not, by itself, demonstrate successful unlearning, as it may reflect degraded generative capabilities rather than the selective removal of those samples' influence.
Second, updates for new requests may weaken earlier deletions, allowing previously deleted samples to reappear in generated outputs.
Figure~\ref{fig:fig1} illustrates both failure modes in existing baselines as deletion requests accumulate.
Continual data unlearning must therefore satisfy each new request while preserving generative utility and maintaining the effectiveness of all prior deletions.

In this paper, we propose a framework for continual data unlearning in diffusion models.
Our key idea is to preserve the response change induced by a completed deletion rather than match the full post-deletion response.
For each request, we record denoiser response changes on the same fixed noisy inputs before and after unlearning.
During subsequent requests, a one-sided penalty discourages reversal along the recorded directions relative to the post-deletion references, while leaving orthogonal response changes and further movement along these directions unpenalized.
To keep memory usage independent of request count, we maintain a fixed-capacity bank of representative transition records.
Records are selected based on how parameter updates locally affect response projections along the recorded directions. This criterion can distinguish records even when their penalties are inactive.

The key contributions of this paper are summarized as follows:
\begin{itemize}
    \item We introduce a transition-based perspective on continual data unlearning, with a one-sided regularizer that discourages reversal of deletion-induced response changes while leaving orthogonal changes and further progress unpenalized.

    \item We develop a fixed-capacity transition bank that selects records based on the local sensitivity of directional progress to parameter updates, even when their penalties are inactive.

    \item We demonstrate a better balance between deletion persistence and generative utility than existing baselines as requests accumulate, using a small transition memory.
\end{itemize}
\section{Preliminaries and Problem Setup}
\label{sec:preliminary}

\subsection{Data Unlearning in Diffusion Models}
\label{sec:datapoint_unlearning}

Data unlearning aims to remove the influence of selected
training examples from a pretrained diffusion model
\citep{alberti2025data}.
We consider a diffusion model $\theta$ pretrained
on a dataset $X = \{x_1, \ldots,x_n\}$ of $n$ datapoints.
A deletion request specifies a subset
$A = \{a_1, \ldots,a_k\} \subset X$
of $k$ datapoints to be unlearned.
The goal is to efficiently update $\theta$ to approximate
the behavior of a model retrained from scratch on the
retained set $R = X \setminus A$.
We consider this goal in both unconditional and
text-to-image diffusion models.
In particular, for text-to-image models, each datapoint
consists of an image and its associated prompt.
Removing a target image's influence should not compromise
the model's ability to generate images aligned with the
associated prompt, provided that the relevant concepts
remain supported by the retained data.

\subsection{Continual Data Unlearning}
\label{sec:continual_unlearning}
In practice, deletion requests may arrive sequentially
rather than as a single batch.
We consider a sequence of requests, where the $t$-th request
provides a target datapoint $a_t^\star \in X$.
Starting from the pretrained model $\theta_0$,
we update $\theta_{t-1}$ to $\theta_t$
upon receiving the $t$-th deletion request.
\emph{Once a request is completed, its raw datapoint is discarded
and is no longer available for subsequent updates}.

To specify the cumulative unlearning objective, we define
\begin{equation}
    A_t = \{a_1^\star, \ldots, a_t^\star\},
    \qquad
    R_t = X \setminus A_t.
    \label{eq:continual_deletion_sets}
\end{equation}
Here, $A_t$ denotes the cumulative deletion set, not a dataset
stored for subsequent updates.
At each stage, the goal is to remove the influence of
the current target while maintaining the effectiveness
of earlier deletions and preserving generative utility,
without accessing previously deleted raw datapoints.
After all $k$ requests, the final model $\epsilon_{\theta_k}$
should therefore approximate the behavior of a model
retrained from scratch on $R_k$, accounting for all
requested deletions rather than only the most recent one.
\section{Transition-Based Regularization for Continual Unlearning}
\label{sec:method}

\begin{figure}[t]
    \centering
    \includegraphics[width=\linewidth]{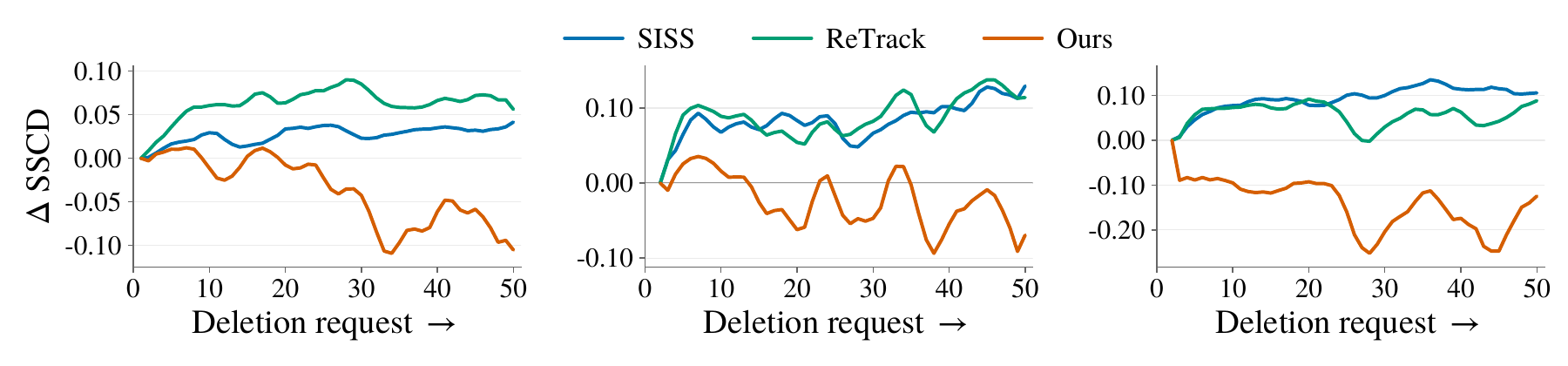}
    \caption{Difference in SSCD for the first unlearned datapoint
across sequential deletion requests on CelebA-HQ.
Each path tracks a different example from an independent deletion sequence. The x-axis denotes the number of deletion requests processed.
The y-axis shows this datapoint's SSCD after each request minus
its SSCD immediately after the first deletion.
Positive values indicate that SSCD is higher than it was
immediately after the first deletion.}

    \label{fig:sscd-increase}
    \label{fig:fig2}
\end{figure}

\paragraph{Motivation.}
Continual data unlearning must preserve generative utility while maintaining the effects of earlier deletions as requests accumulate.
However, repeated unlearning can degrade image quality (Figure~\ref{fig:fig1}), and the SSCD score for the first deletion target can rebound after an initial decline (Figure~\ref{fig:fig2}).
This rebound suggests that later updates can partially reverse the effects of an earlier deletion.

\paragraph{Overview.}
For each deletion request, we record the changes in the model's
denoising responses on fixed probes.
During subsequent requests, bounded gradient corrections discourage
updates that counteract these recorded changes, without requiring
the full responses to remain unchanged.
To limit storage, we select a fixed-size memory of records based on
their local sensitivity to parameter updates.
The overall algorithm is given in Appendix~\ref{app:algorithm}.

\subsection{Recording Response Transitions}
\label{sec:record}
Building on continual-learning methods that regularize future updates using stored model responses~\citep{buzzega2020dark}, we record the response changes induced by each deletion and use them as references for subsequent requests. To characterize these changes, while the target $a_i^\star$ is available, we construct $Q$ fixed evaluation instances, termed \emph{probes}, denoted by $\xi_{i,q}$, $q=1,\ldots,Q$. Each probe consists of a diffusion timestep, a noisy target obtained by forward noising $a_i^\star$ to that timestep, and, when applicable, its conditioning. These components remain fixed throughout all evaluations, such that only the model parameters vary.

Let $\theta_{i-1}$ and $\theta_i$ denote the model parameters immediately before and after processing request $i$, respectively. We denote the corresponding denoiser output for probe $\xi_{i,q}$ by $\Phi(\theta;\xi_{i,q})$.
The corresponding pre- and post-request responses are
\begin{equation}
y_{i,q}^{-}=\Phi(\theta_{i-1};\xi_{i,q}),
\qquad
y_{i,q}^{+}=\Phi(\theta_i;\xi_{i,q}).
\end{equation}

We record
\begin{equation}
d_{i,q}=y_{i,q}^{+}-y_{i,q}^{-},
\qquad
\tau_{i,q}=\langle y_{i,q}^{+},d_{i,q}\rangle,
\label{eq:committed-transition}
\end{equation}
where $\langle\cdot,\cdot\rangle$ denotes the Euclidean inner product.
The vector $d_{i,q}$ captures the change in the probe response
over request $i$.
The scalar $\tau_{i,q}$ is a post-request reference
to detect later reversal along $d_{i,q}$.

The probes and their recorded quantities form the transition record
\begin{equation}
\mathcal T_i
=
\left\{
(\xi_{i,q},d_{i,q},\tau_{i,q})
:
\lVert d_{i,q}\rVert_2>0
\right\}.
\label{eq:transition-record}
\end{equation}
Later checks use $\mathcal T_i$ and the current model without
accessing the raw deleted datapoint or an earlier model checkpoint. Probe construction and storage details are given in Appendix~\ref{app:probe}.

\subsection{Measuring and Penalizing Reversal}
\label{sec:penalty}

Consider a probe whose response changed from $y_{i,q}^{-}$ to
$y_{i,q}^{+}$ during deletion.
We ask whether a later response has moved back along this change,
using the post-request response $y_{i,q}^{+}$ as the reference point.
We measure this movement by the signed margin
\begin{equation}
\begin{aligned}
m_{i,q}(\theta)
&=
\left\langle
\Phi(\theta;\xi_{i,q})-y_{i,q}^{+},
\frac{d_{i,q}}{\lVert d_{i,q}\rVert_2}
\right\rangle
&=
\frac{
\langle \Phi(\theta;\xi_{i,q}),d_{i,q}\rangle-\tau_{i,q}
}{
\lVert d_{i,q}\rVert_2
}.
\end{aligned}
\label{eq:transition-margin}
\end{equation}
The margin is zero immediately after deletion and equals
$-\lVert d_{i,q}\rVert_2$ at the model state immediately before processing request $i$.
A negative margin indicates that the response has moved backward
along the recorded transition, even if it has not returned to
the pre-request response.
The second expression allows this margin to be evaluated without
storing $y_{i,q}^{+}$.

We penalize only the negative margins.
\begin{equation}
G_{i,q}(\theta)
=
[-m_{i,q}(\theta)]_+^2,
\qquad
[x]_+=\max\{x,0\}.
\label{eq:transition-loss}
\end{equation}
Unlike full-response matching, this penalty leaves changes
orthogonal to $d_{i,q}$ and further movement along $d_{i,q}$
unpenalized.

\subsection{Applying Corrections}
\label{sec:loop}

At each optimizer step during request $t$, we first check the next probe
in the newest record $\mathcal T_{t-1}$.
If its margin is negative, indicating reversal, we use the corresponding
penalty gradient to form the correction.
Otherwise, we select an older retained record according to the service-weight schedule described in Section~\ref{sec:compiler} and check its next probe.
We cycle through the probes within each record and apply at most one
correction per step.

Let $L_t^{\mathrm{base}}(\theta)$ denote the ReTrack-based unlearning
objective~\citep{shi2026retrack} for the current request, with gradient
$g^{\mathrm{base}}=\nabla_\theta L_t^{\mathrm{base}}(\theta)$.
For a selected probe $\xi_{i,q}$ with $m_{i,q}(\theta)<0$, we apply
the historical correction
\begin{equation}
g = g^{\mathrm{base}} +
\operatorname{clip}_{\rho\lVert g^{\mathrm{base}}\rVert_2}
\left(\omega\nabla_\theta G_{i,q}(\theta)\right),
\label{eq:bounded-guard-gradient}
\end{equation}
where $\operatorname{clip}_c$ limits the Euclidean norm to at most $c$,
$\omega\geq0$ scales the correction, and $\rho>0$ bounds its norm
relative to the base gradient.
If no checked probe has a negative margin, $g=g^{\mathrm{base}}$.
This clipped soft correction discourages reversal without enforcing a nonnegative margin at every step.

\subsection{Maintaining a Fixed-Size Memory}
\label{sec:compiler}

We keep the newest record $\mathcal T_t$ separately throughout request $t+1$
before considering it for long-term storage.
Let $\mathcal I_t$ index the older records retained after request $t$, with $|\mathcal I_t|\leq K_{\mathrm{mem}}$.
At the end of request $t$, the previous bank and the preceding record
$\mathcal T_{t-1}$ form the candidate set:
\begin{equation}
\mathcal C_t=\mathcal I_{t-1}\cup\{t-1\},
\qquad t\geq2,
\label{eq:compiler-candidate-pool}
\end{equation}
with $\mathcal C_1=\mathcal I_1=\emptyset$.

\paragraph{Selection by local sensitivity.}
With limited memory, we prioritize records whose margins are sensitive
to different update directions rather than retaining many with similar
responses.
However, current penalty gradients vanish for probes with
nonnegative margins and thus cannot distinguish their sensitivity to
later updates, which may still reverse earlier deletions.
We therefore examine the margin gradients
$\nu_{i,q}=\nabla_\theta m_{i,q}(\theta_t)$.
For a small parameter update $\Delta\theta$,
\begin{equation}
m_{i,q}(\theta_t+\Delta\theta)
\approx
m_{i,q}(\theta_t)
+
\langle \nu_{i,q},\Delta\theta\rangle.
\label{eq:local-margin-sensitivity}
\end{equation}
A negative inner product indicates a decrease in the margin,
even when the current penalty is zero.
Similar gradient directions therefore identify margins that tend
to decrease under similar local updates.
We use this similarity to reduce redundancy among retained records.

To compare these directions efficiently, we apply the same
CountSketch-style linear map to all candidate margin gradients $S_t$~\citep{wang2026opus}.
We normalize the sketched gradients to compare directions rather
than magnitudes, then concatenate them in a fixed probe order
to form each record's signature:
\begin{equation}
v_{i,q}=\frac{S_t\nu_{i,q}}{\lVert S_t\nu_{i,q}\rVert_2},
\qquad
u_i=\frac{1}{\sqrt Q}[v_{i,1};\ldots;v_{i,Q}],
\label{eq:boundary-signature}
\end{equation}
where concatenation follows a fixed probe order.

For a nonempty candidate set, we select records so that convex combinations of their signatures approximate all candidate signatures.
\begin{equation}
\mathcal I_t
\in
\arg\min_{\substack{
\mathcal J\subseteq\mathcal C_t\\
|\mathcal J|=\min(K_{\mathrm{mem}},|\mathcal C_t|)
}}
\sum_{i\in\mathcal C_t}
\min_{\eta_i\in\Delta_{\mathcal J}}
\left\lVert
u_i-\sum_{j\in\mathcal J}\eta_{i,j}u_j
\right\rVert_2^2,
\label{eq:main-boundary-cover}
\end{equation}
where $\Delta_{\mathcal J}$ is the simplex over the selected records.
The objective selects signatures that jointly approximate the
candidates' local sensitivity directions.

\paragraph{Allocating future checks.}
Each record carries a positive \emph{service weight} $\mu_i$ for scheduling future checks.
When a candidate record is discarded, we redistribute its weight among the selected records according to its fitted coefficient vector $\eta_i^\star$.
Let $\boldsymbol{\mu}$ collect the selected records'
weights before redistribution:
\begin{equation}
\boldsymbol{\mu}^{+}
=
\boldsymbol{\mu}
+
\sum_{i\in\mathcal C_t\setminus\mathcal I_t}
\mu_i\eta_i^\star.
\label{eq:scheduling-mass-transfer}
\end{equation}
The fitted coefficients sum to one for each discarded record, so this redistribution preserves total service weight.
The normalized weights determine the next request's check schedule, not record selection.

For fixed $K_{\mathrm{mem}}$ and $Q$, we retain at most $K_{\mathrm{mem}}+1$ transition records between requests and compute at most $(K_{\mathrm{mem}}+1)Q$ progress gradients per memory update.
Both bounds are independent of the number of deletion requests.

\section{Experiments}
\label{sec:experiments}

\subsection{Experimental Setup}

\paragraph{Diffusion models and datasets.}
Following prior work on data unlearning in diffusion models~\citep{alberti2025data,shi2026retrack}, we evaluate our method
on unconditional diffusion models for CIFAR-10 and CelebA-HQ, and on Stable Diffusion v1.4 for text-to-image generation.
For the unconditional models, each deletion sequence consists of 50 image deletions.
For Stable Diffusion v1.4, we perform sequential unlearning on 48 memorized prompt--image pairs from SISS and assess general utility using a separate set of 1,000 COCO captions~\citep{lin2014microsoft}.

\paragraph{Evaluation metrics.}
We assess unlearning effectiveness using the Self-Supervised Copy Detection (SSCD) score~\citep{pizzi2022sscd}, which measures object-level similarity between generated images and the corresponding deletion targets.
For unconditional diffusion models, we evaluate generation quality using Fr\'echet Inception Distance (FID)~\citep{heusel2017gans}.
On CelebA-HQ, we additionally compute the negative log-likelihood (NLL) of deletion targets using the exact likelihood formulation
of~\citet{song2021scorebased}, with higher NLL indicating lower likelihood assigned to these targets.
For text-to-image generation, we evaluate image quality using CLIP-IQA~\citep{wang2023clipiqa} and image--text alignment using the
CLIP score~\citep{radford2021learning}, which measures semantic consistency between generated images and their corresponding prompts.
Following SISS, we also report the unlearning success rate.
For each deletion target, we generate 16 samples from the final model using its corresponding prompt.
Unlearning is considered successful for a target if none of these samples is identified as memorized and their mean CLIP-IQA score is at least $0.35$.
To assess general utility, we further report CLIP-IQA and CLIP scores for images generated from the COCO captions.

\paragraph{Baselines.}
We compare our method with Naive (retain-only fine-tuning), NegGrad~\citep{golatkar2020eternal}, EraseDiff~\citep{wu2025erasing}, SISS~\citep{alberti2025data}, and ReTrack~\citep{shi2026retrack}. We match the pretrained initialization, deletion orders, and the number of base optimization steps per request across methods.
Each method uses its own updated model as the starting point for the next request.
For all methods, raw datapoints from completed deletion requests are no longer available for subsequent updates.

\paragraph{Implementation details.}
For CIFAR-10 and CelebA-HQ, we use pretrained diffusion models from \citet{ho2020denoising}.
For SSCD evaluation, we apply forward noising to each target image at diffusion timestep $250$ using a fixed noise, then reconstruct it with 250 DDPM denoising steps, following SISS~\citep{alberti2025data} and ReTrack~\citep{shi2026retrack}.
For Stable Diffusion v1.4, we generate 16 images per prompt variant using 50-step DDIM sampling with a guidance scale of $7.5$ and the same initial noise across methods.
Further details are in Appendix~\ref{app:main-table-protocol}.

\subsection{Continual Unlearning Performance}
\label{sec:results}

\begin{table}[H]
\caption{Results after sequentially unlearning 50 datapoints on CelebA-HQ and CIFAR-10.
We report mean and standard deviation over 3 runs with randomly sampled target sets and deletion orders.}
\label{tab:main-unconditional}
\label{tab:main-cifar50}
\label{tab:main-celeba50}
\centering
\small
\setlength{\tabcolsep}{5pt}
\renewcommand{\arraystretch}{1.12}
\begin{tabular*}{\linewidth}{@{\extracolsep{\fill}}lrrrrr@{}}
\toprule
& \multicolumn{3}{c}{CelebA-HQ}
& \multicolumn{2}{c}{CIFAR-10} \\
\cmidrule(lr){2-4}\cmidrule(l){5-6}
Method
& \multicolumn{1}{c}{SSCD $\downarrow$}
& \multicolumn{1}{c}{FID $\downarrow$}
& \multicolumn{1}{c}{NLL $\uparrow$}
& \multicolumn{1}{c}{SSCD $\downarrow$}
& \multicolumn{1}{c@{}}{FID $\downarrow$} \\
\midrule
Naive
& $0.8595\,{\scriptstyle\pm 0.0017}$
& $20.41\,{\scriptstyle\pm 4.84}$
& $1.3281\,{\scriptstyle\pm 0.0296}$
& $0.5436\,{\scriptstyle\pm 0.0099}$
& $11.38\,{\scriptstyle\pm 1.02}$ \\
NegGrad
& $0.1131\,{\scriptstyle\pm 0.0153}$
& $441.47\,{\scriptstyle\pm 95.44}$
& $18.4781\,{\scriptstyle\pm 1.0963}$
& $0.0539\,{\scriptstyle\pm 0.0507}$
& $431.07\,{\scriptstyle\pm 64.01}$ \\
EraseDiff
& $0.2008\,{\scriptstyle\pm 0.0123}$
& $141.31\,{\scriptstyle\pm 87.25}$
& $5.8061\,{\scriptstyle\pm 0.1735}$
& $0.4567\,{\scriptstyle\pm 0.0110}$
& $33.72\,{\scriptstyle\pm 23.26}$ \\
\midrule
SISS
& $0.6781\,{\scriptstyle\pm 0.0267}$
& $38.80\,{\scriptstyle\pm 31.13}$
& $1.3819\,{\scriptstyle\pm 0.0301}$
& $0.4893\,{\scriptstyle\pm 0.0135}$
& $\mathbf{11.16}\,{\scriptstyle\pm 2.73}$ \\
ReTrack
& $0.4110\,{\scriptstyle\pm 0.0686}$
& $31.48\,{\scriptstyle\pm 9.53}$
& $2.2861\,{\scriptstyle\pm 0.1426}$
& $0.4539\,{\scriptstyle\pm 0.0240}$
& $16.80\,{\scriptstyle\pm 2.56}$ \\
\textbf{Ours}
& $\mathbf{0.3965}\,{\scriptstyle\pm 0.0739}$
& $\mathbf{25.51}\,{\scriptstyle\pm 6.18}$
& $\mathbf{2.3421}\,{\scriptstyle\pm 0.1286}$
& $\mathbf{0.4461}\,{\scriptstyle\pm 0.0253}$
& $16.04\,{\scriptstyle\pm 2.80}$ \\
\bottomrule
\end{tabular*}
\end{table}

\paragraph{Continual unlearning on unconditional models.}
Table~\ref{tab:main-unconditional} reports the results after sequentially
unlearning 50 datapoints on CelebA-HQ and CIFAR-10.
Naive yields relatively low FID but the highest SSCD on both datasets,
indicating greater reconstruction similarity to the deletion targets.
Conversely, NegGrad achieves the lowest SSCD and, on CelebA-HQ, the
highest target NLL, but its FID exceeds 430 on both datasets.
These results highlight that favorable unlearning metrics alone can
coincide with severe degradation in generation quality.
EraseDiff exhibits a similar trade-off on CelebA-HQ, achieving low
SSCD at an FID of 141.31.

With $K_{\mathrm{mem}}=4$, our method improves all reported metrics
over ReTrack on both datasets.
On CelebA-HQ, it reduces FID by 19.0\% relative to ReTrack while
achieving lower SSCD and higher target NLL, and also outperforms
SISS across all metrics.
Although SISS achieves a lower FID on CIFAR-10, our method attains
lower SSCD while maintaining generation quality
comparable to ReTrack.
Overall, these results support a favorable balance between unlearning
effectiveness and generation quality, without the severe quality
degradation observed with NegGrad.


\begin{table}[H]
\caption{Results after sequentially unlearning 48 memorized datapoints from Stable Diffusion v1.4.
We evaluate unlearning performance using prompts corresponding to the deletion targets and general generation utility using COCO captions.
We report the mean and standard deviation over three runs with different deletion orders.}
\label{tab:main-sd48}
\centering
\small
\setlength{\tabcolsep}{2.3pt}
\renewcommand{\arraystretch}{1.12}

\resizebox{\linewidth}{!}{%
\begin{tabular}{@{}lcccccc@{}}
\toprule
& \multicolumn{4}{c}{Unlearning performance}
& \multicolumn{2}{c}{General utility (COCO)} \\
\cmidrule(lr){2-5}
\cmidrule(l){6-7}
Method
& \multicolumn{1}{c}{Success (\%) $\uparrow$}
& \multicolumn{1}{c}{CLIP-IQA $\uparrow$}
& \multicolumn{1}{c}{SSCD $\downarrow$}
& \multicolumn{1}{c}{CLIP score $\uparrow$}
& \multicolumn{1}{c}{CLIP-IQA $\uparrow$}
& \multicolumn{1}{c@{}}{CLIP score $\uparrow$} \\
\midrule

Pretrained
& $0.00$
& $0.8112$
& $0.6548$
& $28.6885$
& $0.8866$
& $26.6049$ \\
\midrule

Naive
& $25.69\,{\scriptstyle\pm 15.91}$
& $0.7169\,{\scriptstyle\pm 0.0365}$
& $0.2521\,{\scriptstyle\pm 0.0133}$
& $32.0518\,{\scriptstyle\pm 0.2843}$
& $0.8871\,{\scriptstyle\pm 0.0031}$
& $27.0232\,{\scriptstyle\pm 0.0699}$ \\

NegGrad
& $0.00\,{\scriptstyle\pm 0.00}$
& $0.1870\,{\scriptstyle\pm 0.0484}$
& $0.0207\,{\scriptstyle\pm 0.0074}$
& $18.2460\,{\scriptstyle\pm 0.5542}$
& $0.2027\,{\scriptstyle\pm 0.0291}$
& $14.0139\,{\scriptstyle\pm 0.3815}$ \\

EraseDiff
& $43.75\,{\scriptstyle\pm 5.51}$
& $0.5312\,{\scriptstyle\pm 0.0492}$
& $0.1226\,{\scriptstyle\pm 0.0224}$
& $27.4282\,{\scriptstyle\pm 1.1531}$
& $0.8887\,{\scriptstyle\pm 0.0081}$
& $26.8076\,{\scriptstyle\pm 0.0649}$ \\
\midrule

SISS
& $87.50\,{\scriptstyle\pm 4.17}$
& $0.6728\,{\scriptstyle\pm 0.0505}$
& $\mathbf{0.1454}\,{\scriptstyle\pm 0.0106}$
& $26.1410\,{\scriptstyle\pm 0.6199}$
& $0.8907\,{\scriptstyle\pm 0.0059}$
& $\mathbf{26.8203}\,{\scriptstyle\pm 0.1474}$ \\

ReTrack
& $91.67\,{\scriptstyle\pm 2.09}$
& $0.7679\,{\scriptstyle\pm 0.0515}$
& $0.2049\,{\scriptstyle\pm 0.0312}$
& $25.8044\,{\scriptstyle\pm 0.7549}$
& $\mathbf{0.8913}\,{\scriptstyle\pm 0.0122}$
& $26.7416\,{\scriptstyle\pm 0.1601}$ \\

\textbf{Ours}
& $\mathbf{98.61}\,{\scriptstyle\pm 1.20}$
& $\mathbf{0.8052}\,{\scriptstyle\pm 0.0349}$
& $0.1485\,{\scriptstyle\pm 0.0189}$
& $\mathbf{26.5527}\,{\scriptstyle\pm 0.3633}$
& $0.8878\,{\scriptstyle\pm 0.0143}$
& $26.7499\,{\scriptstyle\pm 0.2079}$ \\

\bottomrule
\end{tabular}
}
\end{table}

\paragraph{Continual unlearning for text-to-image diffusion.}
Table~\ref{tab:main-sd48} reports the results after sequentially
unlearning 48 memorized datapoints from Stable Diffusion v1.4,
evaluating unlearning on deletion-target prompts and general
utility on COCO.
Naive largely preserves COCO utility but achieves limited
unlearning success.
NegGrad yields the lowest SSCD, yet its success rate is zero,
with severe degradation in image quality and prompt alignment.
EraseDiff also attains low SSCD but achieves only $43.75\%$ success.
These results highlight that low target similarity alone does not
ensure successful unlearning with acceptable image quality.

With $K_{\mathrm{mem}}=2$, our method achieves the highest unlearning success rate of $98.61\%$.
It also achieves lower SSCD and higher CLIP-IQA and CLIP scores than ReTrack on deletion-target prompts.
Although SISS achieves slightly lower SSCD, our method attains a higher success rate and better image quality and prompt alignment.
On COCO, our method’s CLIP-IQA and CLIP scores remain close to the pretrained reference.

\begin{table}[H]
\centering
\small
\caption{Continual unlearning performance on CelebA-HQ across
different deletion sequence lengths and memory capacities.
We report SSCD after 20, 50, and 100 deletions and FID after
100 deletions, with $K_{\mathrm{mem}}\in\{2,3,4\}$ for our method.
Results are the mean and standard deviation over three runs
with different target sets and deletion orders.}
\label{tab:deletion_order_sscd}
\setlength{\tabcolsep}{5pt}
\renewcommand{\arraystretch}{1.12}
\begin{tabular*}{\linewidth}{@{\extracolsep{\fill}}lcccc@{}}
\toprule
& \multicolumn{3}{c}{SSCD $\downarrow$}
& \multicolumn{1}{c@{}}{FID $\downarrow$} \\
\cmidrule(lr){2-4}\cmidrule(l){5-5}
Method & 20 deletions & 50 deletions & 100 deletions & 100 deletions \\
\midrule
SISS
    & $0.6936\,{\scriptstyle \pm\,0.0312}$
    & $0.6826\,{\scriptstyle \pm\,0.0209}$
    & $0.6873\,{\scriptstyle \pm\,0.0103}$
    & $\mathbf{17.70}\,{\scriptstyle \pm\,2.02}$ \\
ReTrack
    & $0.5493\,{\scriptstyle \pm\,0.0354}$
    & $0.4792\,{\scriptstyle \pm\,0.0395}$
    & $0.5030\,{\scriptstyle \pm\,0.0258}$
    & $27.35\,{\scriptstyle \pm\,4.50}$ \\
\midrule
\textbf{Ours} ($K_{\mathrm{mem}}=2$)
    & $\mathbf{0.4900}\,{\scriptstyle \pm\,0.0198}$
    & $0.4497\,{\scriptstyle \pm\,0.0528}$
    & $0.4846\,{\scriptstyle \pm\,0.0254}$
    & $25.98\,{\scriptstyle \pm\,2.13}$ \\
\textbf{Ours} ($K_{\mathrm{mem}}=3$)
    & $0.4921\,{\scriptstyle \pm\,0.0236}$
    & $0.4465\,{\scriptstyle \pm\,0.0370}$
    & $0.4778\,{\scriptstyle \pm\,0.0302}$
    & $22.18\,{\scriptstyle \pm\,1.71}$ \\
\textbf{Ours} ($K_{\mathrm{mem}}=4$)
    & $0.4913\,{\scriptstyle \pm\,0.0236}$
    & $\mathbf{0.4434}\,{\scriptstyle \pm\,0.0377}$
    & $\mathbf{0.4654}\,{\scriptstyle \pm\,0.0347}$
    & $22.73\,{\scriptstyle \pm\,1.30}$ \\
\bottomrule
\end{tabular*}
\end{table}

\subsection{Robustness and Design Analysis}
\label{sec:robustness-design}

We further analyze our method from two perspectives.
For robustness, we study longer deletion sequences under different memory capacities, with additional analyses of deletion-order sensitivity and unlearning persistence provided in Appendix~\ref{app:additional-experiments}.
For design analysis, we examine the choice of response-preservation penalty and memory-selection rule.
Resource costs are also reported in Appendix~\ref{app:resource-cost}.


\paragraph{Longer deletion sequences and memory capacity.}
Table~\ref{tab:deletion_order_sscd} evaluates different deletion
sequence lengths and memory capacities on CelebA-HQ.
Our method achieves lower SSCD than both SISS and ReTrack at
all evaluated sequence lengths, even with $K_{\mathrm{mem}}=2$.
The three capacities yield similar SSCD after 20 deletions,
whereas increasing memory capacity yields progressively lower
SSCD after 50 and 100 deletions.
At 100 deletions, $K_{\mathrm{mem}}=4$ reduces SSCD from
ReTrack's $0.5030$ to $0.4654$, a $7.5\%$ reduction.
All three capacities also achieve lower FID than ReTrack after
100 deletions, showing that improved target suppression is
accompanied by better generation quality.
Overall, our method maintains its advantage over longer deletion
sequences with a small, fixed memory capacity, while additional
memory provides greater benefits as the sequence length increases.

\begin{table}[H]
\caption{Comparison of response-preservation penalties on
CelebA-HQ and Stable Diffusion v1.4 after 50 and 48 deletions,
respectively.
All variants use $K_{\mathrm{mem}}=2$.
We report the mean and standard deviation over three runs
with different deletion orders.}
\label{tab:ablation-penalty}
\centering
\small
\setlength{\tabcolsep}{3pt}
\renewcommand{\arraystretch}{1.12}

\resizebox{\linewidth}{!}{%
\begin{tabular}{@{}lrrrrrr@{}}
\toprule
& \multicolumn{2}{c}{CelebA-HQ}
& \multicolumn{4}{c@{}}{Stable Diffusion v1.4} \\
\cmidrule(lr){2-3}\cmidrule(l){4-7}
Penalty
& \multicolumn{1}{c}{SSCD $\downarrow$}
& \multicolumn{1}{c}{FID $\downarrow$}
& \multicolumn{1}{c}{Success (\%) $\uparrow$}
& \multicolumn{1}{c}{CLIP-IQA $\uparrow$}
& \multicolumn{1}{c}{SSCD $\downarrow$}
& \multicolumn{1}{c@{}}{CLIP score $\uparrow$} \\
\midrule
Full-output
& $0.4326\,{\scriptstyle\pm 0.0404}$
& $33.69\,{\scriptstyle\pm 9.26}$
& $92.36\,{\scriptstyle\pm 1.20}$
& $0.8013\,{\scriptstyle\pm 0.0309}$
& $0.1155\,{\scriptstyle\pm 0.0168}$
& $\mathbf{26.6624}\,{\scriptstyle\pm 0.2519}$ \\
Symmetric
& $0.4391\,{\scriptstyle\pm 0.0279}$
& $30.03\,{\scriptstyle\pm 10.79}$
& $93.75\,{\scriptstyle\pm 3.61}$
& $0.7937\,{\scriptstyle\pm 0.0299}$
& $\mathbf{0.1061}\,{\scriptstyle\pm 0.0190}$
& $26.4982\,{\scriptstyle\pm 0.1673}$ \\
\textbf{Ours}
& $\mathbf{0.3975}\,{\scriptstyle\pm 0.0680}$
& $\mathbf{27.62}\,{\scriptstyle\pm 9.03}$
& $\mathbf{98.61}\,{\scriptstyle\pm 1.20}$
& $\mathbf{0.8052}\,{\scriptstyle\pm 0.0349}$
& $0.1485\,{\scriptstyle\pm 0.0189}$
& $26.5527\,{\scriptstyle\pm 0.3633}$ \\
\bottomrule
\end{tabular}%
}
\end{table}

\paragraph{Choice of response-preservation penalty.}
We compare three response-preservation penalties (Table~\ref{tab:ablation-penalty}).
Full-output preservation penalizes all deviations from the post-deletion response, whereas symmetric directional preservation penalizes both reversal and further movement along the recorded deletion direction. Our one-sided penalty penalizes only reversal, allowing movement beyond the post-deletion reference and orthogonal changes.
On CelebA-HQ, our one-sided penalty achieves the lowest SSCD and FID among the three variants.
On Stable Diffusion v1.4, it achieves the highest success rate of $98.61\%$, compared with $92.36\%$ for full-output preservation and $93.75\%$ for symmetric directional preservation.
It also yields the highest CLIP-IQA while maintaining comparable prompt alignment.
Although the other penalties attain lower SSCD, this does not translate into higher unlearning success.
Overall, these results support penalizing only reversal rather than imposing broader response constraints in the evaluated settings.

\begin{table}[H]
\caption{Comparison of memory-selection rules on CelebA-HQ
and Stable Diffusion v1.4 after 50 and 48 deletions, respectively.
All variants use $K_{\mathrm{mem}}=2$.
We report the mean and standard deviation over three runs
with different deletion orders.}
\label{tab:ablation-memory}
\centering
\small
\setlength{\tabcolsep}{2pt}
\renewcommand{\arraystretch}{1.12}
\begin{tabular*}{\linewidth}{@{\extracolsep{\fill}}lrrrrrr@{}}
\toprule
& \multicolumn{2}{c}{CelebA-HQ}
& \multicolumn{4}{c@{}}{Stable Diffusion v1.4} \\
\cmidrule(lr){2-3}\cmidrule(l){4-7}
Selection
& \multicolumn{1}{c}{SSCD $\downarrow$}
& \multicolumn{1}{c}{FID $\downarrow$}
& \multicolumn{1}{c}{Success (\%) $\uparrow$}
& \multicolumn{1}{c}{CLIP-IQA $\uparrow$}
& \multicolumn{1}{c}{SSCD $\downarrow$}
& \multicolumn{1}{c@{}}{CLIP $\uparrow$} \\
\midrule
FIFO
& $0.4400\,{\scriptstyle\pm 0.0269}$
& $32.78\,{\scriptstyle\pm 10.48}$
& $93.06\,{\scriptstyle\pm 3.18}$
& $0.7907\,{\scriptstyle\pm 0.0293}$
& $\mathbf{0.1205}\,{\scriptstyle\pm 0.0150}$
& $26.2898\,{\scriptstyle\pm 1.0175}$ \\
Random
& $0.4379\,{\scriptstyle\pm 0.0288}$
& $31.39\,{\scriptstyle\pm 9.26}$
& $93.75\,{\scriptstyle\pm 5.51}$
& $0.7916\,{\scriptstyle\pm 0.0458}$
& $0.1216\,{\scriptstyle\pm 0.0111}$
& $\mathbf{27.1066}\,{\scriptstyle\pm 0.4785}$ \\
\textbf{Ours}
& $\mathbf{0.3975}\,{\scriptstyle\pm 0.0680}$ 
& $\mathbf{27.62}\,{\scriptstyle\pm 9.03}$ 
& $\mathbf{98.61}\,{\scriptstyle\pm 1.20}$
& $\mathbf{0.8052}\,{\scriptstyle\pm 0.0349}$
& $0.1485\,{\scriptstyle\pm 0.0189}$
& $26.5527\,{\scriptstyle\pm 0.3633}$ \\
\bottomrule
\end{tabular*}
\end{table}

\paragraph{Choice of memory-selection rule.}
We compare our gradient-signature selection with FIFO (first-in, first-out) and Random
(Table~\ref{tab:ablation-memory}).
Each rule selects up to two records from the previous memory
and the preceding deletion record.
FIFO retains the most recent records, whereas Random samples
uniformly without replacement from this candidate set.
Previously discarded records are not reconsidered.

On CelebA-HQ, our selection achieves the lowest SSCD and FID,
improving both target suppression and generation quality.
On Stable Diffusion v1.4, it achieves the highest success rate
of $98.61\%$, compared with $93.06\%$ for FIFO and $93.75\%$
for Random, together with the highest CLIP-IQA.
Its success rate also exhibits the smallest standard deviation
across deletion orders.
These results support the effectiveness of gradient-signature
selection over recency-based or random selection, enabling
better unlearning performance with a small, fixed memory budget.

\section{Related Work}
\label{sec:relatedwork}

\paragraph{Unlearning in Diffusion Models.}
Unlearning in diffusion models has been studied at both the concept and datapoint levels.
Concept unlearning methods fine-tune pretrained diffusion models to suppress target concepts while preserving non-target generative capabilities
\citep{gandikota2023erasing,kumari2023ablating,heng2023selective,zhang2024forget}.
Other approaches employ saliency-guided parameter updates \citep{fan2024salun},
trainable concept-specific adapters \citep{lyu2024one},
or LoRA fine-tuning for multi-concept erasure \citep{lu2024mace}.
Further work incorporates adversarial training to improve the robustness of concept erasure against adversarial prompts
\citep{kim2024race,zhang2024defensive,srivatsan2025stereo}.
In contrast, data unlearning aims to remove the influence of specific training examples without suppressing the broader concepts they represent.
Representative approaches include EraseDiff's forgetting--retention optimization \citep{wu2025erasing}, SISS's importance-sampled objective \citep{alberti2025data}, ReTrack's redirection of target denoising trajectories toward retained examples \citep{shi2026retrack}.

\paragraph{Continual Unlearning in Diffusion Models.}
In practical deployments, deletion requests may arrive sequentially rather than as a single, predefined batch.
However, successful unlearning need not persist under subsequent model updates: previously unlearned concepts can re-emerge after fine-tuning, even on benign data \citep{George_2025_CVPR,suriyakumar2026unstable}.
This instability motivates continual unlearning, which requires earlier deletions to remain effective as new requests are processed.
While Meta-Unlearning \citep{gao2025meta} seeks to make concept erasure resistant to downstream fine-tuning, recent studies explicitly address sequential concept deletion.
\citet{lee2026continual} investigate regularization and semantic-aware gradient projection to mitigate cumulative parameter drift,
while LACU \citep{lacu} employs locality-aware replacement targets and replay to protect neighboring retained concepts.
The same challenge arises in data unlearning, where each new request must be satisfied while maintaining the effectiveness of earlier deletions and preserving generative utility.
Unlike concept-level suppression, this requires removing individual samples' influence without suppressing their associated concepts, provided that these concepts remain supported by retained data.
\section{Limitations}

Our framework uses a fixed-capacity transition bank with a memory budget set in advance. Adaptive capacity allocation could improve flexibility when the number of deletion requests is unknown. Each transition is represented by a finite set of probes, which may not capture the deletion-induced response change exhaustively. Maintaining the bank also adds memory and computation, although these costs remain bounded for fixed capacity and probe count.

\section{Conclusion}

We presented a transition-based framework for sustaining data unlearning as deletion requests accumulate. Across unconditional and text-to-image diffusion models, our method strengthens deletion persistence while preserving generative utility. On Stable Diffusion, it achieves the highest unlearning success among the evaluated methods, with high-quality, prompt-aligned generation and general utility close to the pretrained reference. These gains are achieved with a small, fixed transition memory whose storage does not grow with the number of deletion requests. Overall, our results support preventing the reversal of deletion-induced response changes, rather than matching full post-deletion responses, as an effective approach to continual unlearning.
\newpage
\bibliography{references}
\bibliographystyle{iclr2027_conference}
\clearpage

\appendix
\section*{Appendix}
\paragraph{Appendix Overview.}
This appendix provides additional details and analyses supporting the main paper.
Appendix~\ref{app:fig_1_prompt} lists the prompts used for the qualitative results in Figure~\ref{fig:fig1}.
Appendix~\ref{app:algorithm} presents the overall continual unlearning algorithm,
followed by probe construction and stored-state details in Appendix~\ref{app:probe}
and implementation details in Appendix~\ref{app:training-details}.
Appendix~\ref{app:additional-experiments} provides additional experiments on
request-order sensitivity and unlearning persistence, while
Appendix~\ref{app:resource-cost} reports memory and runtime costs.
Finally, Appendix~\ref{app:local-analysis} provides a local analysis of the proposed
reversal regularization.

\section{Prompts for Figure \ref{fig:fig1}}
\label{app:fig_1_prompt}
The following prompts were used to generate the samples shown in Figure \ref{fig:fig1}:

\begin{itemize}

    \item Aaron Paul to Play Luke Skywalker at LACMA Reading of <i>The Empire Strikes Back</i>

    \item Video: Charlize Theron in Trailer for New Gillian Flynn Adaptation, <i>Dark Places</i>

    \item <em>Bloodborne</em> Video: Sony Explains the Game's Procedurally Generated Dungeons
    
    \item Mothers influence on her young hippo
\end{itemize}
\Needspace{36\baselineskip}
\section{Algorithm Overview}
\label{app:algorithm}

\begin{algorithm}[H]
\caption{Continual data unlearning: overall procedure}
\label{alg:transition-overall-defined}
\begingroup
\small
\newcommand{\tv}[1]{\mathsf{#1}}
\algrenewcommand\algorithmicindent{1.15em}
\algrenewcommand\algorithmiccomment[1]{\hfill{\footnotesize\itshape #1}}
\begin{algorithmic}[1]
\Require $\theta$: pretrained denoiser parameters, $\mathcal X$: original training data
\Require $N$: steps per request, $Q$: probes per target, $K$: maximum older records
\Statex \textbf{Probe:} a fixed noisy input, its timestep, and optional conditioning.
\Statex \textbf{Record:} probes, response-change directions, post-deletion references, and a check-frequency weight.
\Statex \textbf{Reversal:} backward movement from a post-deletion reference along its recorded direction.
\Statex
\State $\tv{retained}\gets\mathcal X$
\State $\tv{newest}\gets\varnothing$ \Comment{record of the most recent deletion}
\State $\tv{bank}\gets\varnothing$ \Comment{older-record memory with capacity $K$}
\For{each incoming deletion target $a$}
    \State $\tv{retained}\gets\tv{retained}\setminus\{a\}$
    \State $\tv{probes}\gets Q$ fixed probes constructed from $a$
    \State $\tv{before}\gets$ denoiser outputs under $\theta$ on $\tv{probes}$
    \Statex
    \For{$N$ optimizer steps}
        \State $\tv{base}\gets$ ReTrack gradient for $(\theta,a,\tv{retained})$
        \State $\tv{correction}\gets\Call{CorrectReversal}{\theta,\tv{newest},\tv{bank},\tv{base}}$
        \Statex \hspace{2.3em}\textit{Bounded penalty gradient, or zero when no reversal is detected.}
        \State $\theta\gets\operatorname{OptimizerStep}(\theta,\tv{base}+\tv{correction})$
    \EndFor
    \Statex
    \State $\tv{after}\gets$ denoiser outputs under updated $\theta$ on the same $\tv{probes}$
    \State $\tv{record}\gets\Call{RecordTransition}{\tv{probes},\tv{before},\tv{after}}$
    \State Move $\tv{newest}$ into $\tv{bank}$, if nonempty
    \State $\tv{bank}\gets\Call{UpdateMemory}{\theta,\tv{bank},K}$ \Comment{retain up to $K$ representative records}
    \State $\tv{newest}\gets\tv{record}$ \Comment{checked first during the next request}
    \State Discard raw target $a$, $\tv{before}$, and $\tv{after}$
\EndFor
\State \Return $\theta$ \Comment{model after all deletion requests}
\end{algorithmic}
\endgroup
\end{algorithm}

\newpage

\begin{algorithm}[H]
\caption{Record the response transition after a deletion}
\label{alg:transition-operations-readable}
\label{alg:transition-record}
\begingroup
\small
\newcommand{\tv}[1]{\mathsf{#1}}
\algrenewcommand\algorithmicindent{1.15em}
\algrenewcommand\algorithmiccomment[1]{\hfill{\footnotesize\itshape #1}}
\begin{algorithmic}[1]
\Require $w_0>0$: initial check weight
\Statex $\operatorname{Project}(y,d)=\langle y,d\rangle/\lVert d\rVert_2$, for response $y$ and nonzero change $d$.
\Statex
\Function{RecordTransition}{$\tv{probes},\tv{before},\tv{after}$}
    \State $\tv{record}\gets\varnothing$
    \For{each $\tv{probe}\in\tv{probes}$ with a nonzero response change}
        \State $\tv{direction}\gets\tv{after}[\tv{probe}]-\tv{before}[\tv{probe}]$
        \State $\tv{reference}\gets\operatorname{Project}(\tv{after}[\tv{probe}],\tv{direction})$
        \State Append $(\tv{probe},\tv{direction},\tv{reference})$ to $\tv{record}$
    \EndFor
    \State \Return $\tv{record}$ with check weight $w_0$ \Comment{future check frequency}
\EndFunction
\end{algorithmic}
\endgroup
\end{algorithm}

\begin{algorithm}[H]
\caption{Detect reversal and compute a bounded correction}
\label{alg:transition-correct}
\begingroup
\small
\newcommand{\tv}[1]{\mathsf{#1}}
\algrenewcommand\algorithmicindent{1.15em}
\algrenewcommand\algorithmiccomment[1]{\hfill{\footnotesize\itshape #1}}
\begin{algorithmic}[1]
\Require $\omega\geq0$: correction scale, $\rho>0$: norm-cap ratio
\Statex $\Phi_\theta$: model response, $\theta$: model parameters, $g$: base gradient.
\Statex $\operatorname{Project}(y,d)=\langle y,d\rangle/\lVert d\rVert_2$ for nonzero $d$.
\Statex
\Function{CorrectReversal}{$\theta,\tv{newest},\tv{bank},g$}
    \For{$\tv{record}$: $\tv{newest}$, then one weight-scheduled record from $\tv{bank}$}
        \State Skip missing or empty records
        \State $(\tv{probe},\tv{direction},\tv{reference})\gets$ next cyclic entry of $\tv{record}$
        \State $\tv{current}\gets\operatorname{Project}(\Phi_\theta(\tv{probe}),\tv{direction})$
        \State $\tv{reversal}\gets\max(\tv{reference}-\tv{current},\,0)$
        \If{$\tv{reversal}>0$}
            \State $h\gets\omega\nabla_\theta(\tv{reversal}^{2})$
            \State Rescale $h$ if needed to have norm at most $\rho\lVert g\rVert_2$
            \State \Return $h$
        \EndIf
    \EndFor
    \State \Return $0$
\EndFunction
\end{algorithmic}
\endgroup
\end{algorithm}

\begin{algorithm}[H]
\caption{Select representative records and redistribute check weights}
\label{alg:transition-memory}
\begingroup
\small
\newcommand{\tv}[1]{\mathsf{#1}}
\algrenewcommand\algorithmicindent{1.15em}
\algrenewcommand\algorithmiccomment[1]{\hfill{\footnotesize\itshape #1}}
\begin{algorithmic}[1]
\Statex $\theta$: model parameters, $K$: maximum number of older records.
\Statex
\Function{UpdateMemory}{$\theta,\tv{bank},K$}
    \If{$|\tv{bank}|\leq K$} \State \Return $\tv{bank}$ \EndIf
    \For{each record in $\tv{bank}$}
        \State $\tv{gradients}\gets$ gradients of all current probe projections w.r.t.\ $\theta$
        \State $\tv{signature}\gets$ concatenate compressed, normalized gradients in probe order
    \EndFor
    \For{each record in $\tv{bank}$}
        \State Fit its signature by a weighted average of the other signatures
    \EndFor
    \State Discard the record with the smallest squared fitting error
    \State Redistribute its check weight using its fitted averaging weights
    \State \Return $\tv{bank}$
\EndFunction
\end{algorithmic}
\endgroup
\end{algorithm}

\clearpage
\section{Probe Construction and Storage}
\label{app:probe}
\label{app:probe-implementation}

\paragraph{Response construction.}
Each probe is constructed from the deletion target using a fixed diffusion
timestep, noise realization, and, when applicable, conditioning. Writing
$h$ for the identity in pixel-space diffusion and the fixed VAE encoding
in latent diffusion,
\[
z_s(x,\epsilon)=\gamma_s h(x)+\sigma_s\epsilon,\qquad
\Phi(\theta,\xi)=
\epsilon_\theta(z_s^{\mathrm{tar}},s,c).
\]
Here $\gamma_s$ and $\sigma_s$ are the forward-process signal and noise
coefficients, and conditioning $c$ is omitted for unconditional models.
The noisy input, timestep, and conditioning remain fixed across
pre-deletion, post-deletion, and subsequent evaluations. Thus, changes
in the recorded response arise from model updates rather than newly
sampled probe inputs.

\paragraph{Probe sampling.}
We use $Q=4$ probes per record. For CIFAR-10 and CelebA-HQ, timesteps
are sampled uniformly from $\{200,\ldots,999\}$, with an independently
sampled Gaussian noise realization for each probe. Stable Diffusion uses
timesteps $\{225,375,625,875\}$ and the original target prompt, with a
fixed Gaussian noise realization at each timestep. These training probes
are distinct from the timestep-250 reconstructions and prompt-conditioned
samples used to evaluate unlearning.

\paragraph{Persistent state.}
The pixel-space implementations store the noisy target inputs, their
timesteps, the displacement $d$, and the scalar post-deletion projection
$\tau$ for each probe. CIFAR-10 retains these floating-point tensors in
FP32. CelebA-HQ stores noisy inputs and displacements in FP16, with
projections computed in FP32 using the stored displacements.
The separate pre- and post-deletion responses are not needed for later
checks. Stable Diffusion instead stores the noisy target latents,
text-conditioning embeddings, timesteps, pre-deletion responses $y^-$, and displacements $d$ in FP32. 

Each retained record also carries a service weight and a cyclic probe
cursor. The newest record is kept separately for one request before
entering memory selection, giving at most $K_{\mathrm{mem}}+1$ records
between requests. Full parameter gradients and their sketches are
temporary computations, not persistent historical records.

\clearpage
\section{Implementation Details}
\label{app:main-table-protocol}
\label{app:training-details}

\begin{table}[H]
\centering
\small
\caption{Training and correction hyperparameters. $K_{\mathrm{mem}}$
counts older records and excludes the separately stored newest record.
Memory-capacity and penalty variants override the corresponding entries.}
\label{tab:training-settings}
\setlength{\tabcolsep}{5pt}
\renewcommand{\arraystretch}{1.12}
\begin{tabular*}{\linewidth}{@{\extracolsep{\fill}}lccc@{}}
\toprule
Parameter & CIFAR-10 & CelebA-HQ & Stable Diffusion v1.4 \\
\midrule
Updates per request & $60$ & $40$ & $35$ \\
Learning rate & $2\times10^{-5}$ & $5\times10^{-6}$ & $10^{-5}$ \\
AdamW $(\beta_1,\beta_2)$ & $(0.95,0.999)$ & $(0.95,0.999)$ & $(0.9,0.999)$ \\
Weight decay & $10^{-6}$ & $10^{-6}$ & $10^{-2}$ \\
AdamW $\epsilon$ & $10^{-8}$ & $10^{-8}$ & $10^{-8}$ \\
Probes $Q$ & $4$ & $4$ & $4$ \\
Memory capacity $K_{\mathrm{mem}}$ & $4$ & $4$ & $2$ \\
Correction scale $\omega$ & $1$ & $1$ & $0.5$ \\
Relative norm cap $\rho$ & $0.20$ & $0.10$ & $0.02$ \\
\bottomrule
\end{tabular*}
\end{table}

\paragraph{Optimization.}
Table~\ref{tab:training-settings} summarizes the ReTrack-based training
and correction settings.
We update the denoiser in FP32 and carry its parameters across deletion
requests, while reinitializing AdamW for each request. Learning rates
are constant, with no warm-up. For Stable Diffusion, the entire UNet is
updated while the VAE and text encoder remain fixed.
The correction is added to the base gradient before the optimizer step.
Its norm is bounded relative to the base gradient, rather than the
parameter update produced by AdamW.

\paragraph{Base-update settings.}
CIFAR-10 uses batches of 128 with one gradient-accumulation step.
The CelebA-HQ ReTrack-based runs use batches of four with two
accumulation steps. SISS uses batches of four with four
accumulation steps.
ReTrack uses ten neighbors, with objective coefficient $0.02$ on CIFAR-10,
$0.005$ on CelebA-HQ, and $0.5$ on Stable Diffusion.
SISS uses mixture coefficient $0.5$ and rescaling factors 5 and 500
on CIFAR-10 and CelebA-HQ, respectively.
Previously deleted targets are excluded from positive supervision
throughout each sequence.

\paragraph{Gradient signatures and selection.}
We use two CountSketch-style repeats of dimension 8,192 with sketch seed
27182818, concatenated and scaled by $1/\sqrt{2}$. Both repeats use a
shared parameter ordering and bucket assignment across records, with
independent random signs. Each probe sketch is normalized separately.
The four normalized sketches are concatenated and scaled by $1/\sqrt{Q}$.
The implementation differentiates normalized progress
$p=1+m/\lVert d\rVert_2$. Its gradient is a positive scalar multiple of
the margin gradient, so the normalized sketch direction is unchanged.

The candidate pool contains at most $K_{\mathrm{mem}}+1$ records.
We enumerate subsets of the required size and solve the simplex-constrained
least-squares fit for each candidate by enumerating active coefficient sets.
The inner solves use FP64 and a feasibility tolerance of $10^{-8}$.
Discarded records transfer their service weights according to the fitted
coefficients. Largest-remainder allocation, with at least one scheduled
visit per retained record, converts these weights into a finite checking
schedule. Newest-record checks retain priority over this older-record
schedule.
\clearpage
\section{Additional Experiments}
\label{app:additional-experiments}

\subsection{Sensitivity to Request Order}
\label{app:order-sensitivity}

\begin{table}[H]
\caption{Sensitivity to request order on CelebA-HQ and Stable Diffusion v1.4.
For each setting, we evaluate five different deletion orders using
a fixed target set: 50 targets for CelebA-HQ and 30 targets for
Stable Diffusion.
We report the mean and standard deviation across these orders.
Our method uses $K_{\mathrm{mem}}=2$ in both settings.}
\label{tab:order-sensitivity}
\centering
\small
\setlength{\tabcolsep}{2pt}
\renewcommand{\arraystretch}{1.12}

\begin{tabular*}{\linewidth}{@{\extracolsep{\fill}}lrrrrrr@{}}
\toprule
& \multicolumn{2}{c}{CelebA-HQ}
& \multicolumn{4}{c@{}}{Stable Diffusion v1.4} \\
\cmidrule(lr){2-3}\cmidrule(l){4-7}

Method
& \multicolumn{1}{c}{SSCD $\downarrow$}
& \multicolumn{1}{c}{FID $\downarrow$}
& \multicolumn{1}{c}{Success (\%) $\uparrow$}
& \multicolumn{1}{c}{CLIP-IQA $\uparrow$}
& \multicolumn{1}{c}{SSCD $\downarrow$}
& \multicolumn{1}{c@{}}{CLIP score $\uparrow$} \\
\midrule

SISS
& $0.6985\,{\scriptstyle\pm 0.0256}$
& $\mathbf{25.06}\,{\scriptstyle\pm 4.78}$
& $74.67\,{\scriptstyle\pm 14.64}$
& $0.6782\,{\scriptstyle\pm 0.0574}$
& $0.1611\,{\scriptstyle\pm 0.0218}$
& $26.2896\,{\scriptstyle\pm 0.8763}$ \\

ReTrack
& $0.4292\,{\scriptstyle\pm 0.0228}$
& $29.01\,{\scriptstyle\pm 2.43}$
& $85.33\,{\scriptstyle\pm 9.31}$
& $0.7942\,{\scriptstyle\pm 0.0346}$
& $0.2286\,{\scriptstyle\pm 0.0163}$
& $26.7059\,{\scriptstyle\pm 0.7214}$ \\

\midrule
\textbf{Ours}
& $\mathbf{0.4115}\,{\scriptstyle\pm 0.0187}$
& $28.90\,{\scriptstyle\pm 3.24}$
& $\mathbf{90.67}\,{\scriptstyle\pm 2.79}$
& $\mathbf{0.7969}\,{\scriptstyle\pm 0.0171}$
& $\mathbf{0.1545}\,{\scriptstyle\pm 0.0250}$
& $\mathbf{26.7712}\,{\scriptstyle\pm 0.5602}$ \\

\bottomrule
\end{tabular*}
\end{table}

We assess sensitivity to request order by varying the deletion
sequence while keeping the target set fixed
(Table~\ref{tab:order-sensitivity}).
On CelebA-HQ, our method achieves the lowest SSCD with the
smallest standard deviation, while also attaining a slightly
lower FID than ReTrack.
On Stable Diffusion v1.4, it achieves the highest success rate with
substantially lower variability, reaching $(90.67\pm2.79)\%$
compared with $(85.33\pm9.31)\%$ for ReTrack and
$(74.67\pm14.64)\%$ for SISS.
It also achieves lower SSCD and higher CLIP-IQA and CLIP scores
than both baselines, with smaller standard deviations in the
two CLIP-based metrics.
Overall, our method improves both SSCD on CelebA-HQ and
unlearning success on Stable Diffusion v1.4, while reducing their
variability across deletion orders and maintaining competitive
generation quality.

\subsection{Persistence of Unlearning}
\label{app:unlearning-persistence}

\begin{table}[H]
\caption{SSCD rebound on CelebA-HQ over three runs of 50 deletions.
Ours uses $K_{\mathrm{mem}}=4$.
The lower panel requires immediate SSCD $\leq0.5$ under all methods.}
\label{tab:celeb-sscd-rebound}
\centering
\small
\setlength{\tabcolsep}{4pt}
\renewcommand{\arraystretch}{1.12}
\begin{tabular*}{\linewidth}{@{\extracolsep{\fill}}lrrr@{}}
\toprule
& \multicolumn{3}{c@{}}{SSCD $\downarrow$} \\
\cmidrule(l){2-4}
Method
& \multicolumn{1}{c}{Immediately after deletion}
& \multicolumn{1}{c}{After 50 deletions}
& \multicolumn{1}{c@{}}{Mean rebound} \\
\midrule
\multicolumn{4}{l}{All evaluated targets } \\
SISS & $0.6069\,{\scriptstyle\pm 0.0044}$
& $0.6814\,{\scriptstyle\pm 0.0093}$ & $\mathbf{0.0746}\,{\scriptstyle\pm 0.0136}$ \\
ReTrack & $0.3631\,{\scriptstyle\pm 0.0237}$
& $0.4810\,{\scriptstyle\pm 0.0091}$ & $0.1210\,{\scriptstyle\pm 0.0151}$ \\
\textbf{Ours} & $\mathbf{0.3529}\,{\scriptstyle\pm 0.0153}$
& $\mathbf{0.4408}\,{\scriptstyle\pm 0.0326}$ & $0.1001\,{\scriptstyle\pm 0.0325}$ \\
\midrule
\multicolumn{4}{l}{Common targets with low post-deletion SSCD } \\
SISS & $0.3913\,{\scriptstyle\pm 0.0293}$
& $0.5232\,{\scriptstyle\pm 0.0731}$ & $0.1319\,{\scriptstyle\pm 0.0438}$ \\
ReTrack & $\mathbf{0.1831}\,{\scriptstyle\pm 0.0245}$
& $0.3094\,{\scriptstyle\pm 0.0409}$ & $0.1269\,{\scriptstyle\pm 0.0190}$ \\
\textbf{Ours} & $0.2127\,{\scriptstyle\pm 0.0097}$
& $\mathbf{0.3009}\,{\scriptstyle\pm 0.0152}$ & $\mathbf{0.0919}\,{\scriptstyle\pm 0.0182}$ \\
\bottomrule
\end{tabular*}
\end{table}

We measure persistence by averaging positive SSCD increases from
immediately after deletion to the final model under matched evaluation
conditions, counting decreases as zero (Table~\ref{tab:celeb-sscd-rebound}).
Our method reduces rebound relative to ReTrack ($0.1210\to0.1001$)
while achieving lower immediate and final SSCD.
SISS has the smallest rebound but much higher SSCD at both time points;
low rebound alone therefore does not imply effective unlearning.

A small rebound can also reflect weak initial unlearning, as illustrated
by SISS's high immediate SSCD. We therefore additionally compare common
targets with immediate SSCD $\leq0.5$ under all methods. On this subset,
our method reduces rebound by $27.6\%$ relative to ReTrack and achieves
lower final SSCD despite starting from higher post-deletion similarity.
This supports improved persistence beyond stronger initial suppression
on the evaluated subset.

\section{Memory and Runtime}
\label{app:resource-cost}

We profile ReTrack and our method on the first ten requests of three
CelebA-HQ deletion sequences, using 40 optimizer steps per request, $Q=4$,
and $\rho=0.20$. Each configuration starts from the same pretrained model
and maintains its own evolving parameters and records. Measurements use a Blackwell GPU, 12 CPUs, four computation threads,
and FP32 with TF32 disabled. All runs use the same GPU model.

Table~\ref{tab:resource-cost} summarizes requests 7--10, when all memory
capacities are filled. We first average within each sequence and then report
the mean and sample standard deviation across three sequences. Request time
includes training, neighbor search, process startup, model loading,
checkpoint I/O, transition recording, and memory selection, but excludes
evaluation. GPU memory is the mean of per-request peak PyTorch allocations
during training, including optimizer and correction state. Memory used outside the training steps is not included.

\begin{table}[htbp]
\centering
\small
\caption{CelebA-HQ resource costs across three deletion sequences.}
\label{tab:resource-cost}
\setlength{\tabcolsep}{6pt}
\begin{tabular}{@{}lrrr@{}}
\toprule
Method & Time/request (s) & Relative time & Training peak (MiB) \\
\midrule
ReTrack & $224.47\pm19.38$ & $1.00\times$ & $21317.48\pm0.00$ \\
Ours ($K_{\mathrm{mem}}=2$) & $277.42\pm16.23$ & $1.24\times$ & $21331.51\pm1.16$ \\
Ours ($K_{\mathrm{mem}}=3$) & $272.15\pm6.97$ & $1.22\times$ & $21331.90\pm0.67$ \\
Ours ($K_{\mathrm{mem}}=4$) & $285.41\pm5.37$ & $1.28\times$ & $21332.28\pm0.67$ \\
\bottomrule
\end{tabular}
\end{table}

Relative time is the mean ratio to ReTrack over matched deletion sequences.
The measured request-time overhead is approximately 22--28\%, while
training peak allocation increases by approximately 14--15 MiB.
Timing includes startup and I/O costs, and no unusually slow training
steps are excluded. Runtime does not increase monotonically with capacity
in these measurements, so the small differences between capacities should
not be interpreted as a consistent speed advantage.

\clearpage
\section{Local Analysis of Reversal Regularization}
\label{app:local-analysis}

\paragraph{Notation.}
Fix one probe and its recorded quantities $d\neq0$ and $y^+$.
Omitting request and probe indices, write $y(\theta)$ for the response and
define
\[
\hat d=\frac{d}{\lVert d\rVert_2},\qquad
m(\theta)=\langle y(\theta)-y^+,\hat d\rangle,\qquad
G(\theta)=[-m(\theta)]_+^2.
\]

\begin{proposition}[Response-space distance]
\label{prop:response-distance}
Let $\mathcal H=\{y:\langle y-y^+,\hat d\rangle\geq0\}$.
Then
\begin{equation}
G(\theta)=\operatorname{dist}(y(\theta),\mathcal H)^2,
\qquad
\Pi_{\mathcal H}(y(\theta))=y(\theta)+[-m(\theta)]_+\hat d,
\label{eq:app-halfspace-distance}
\end{equation}
where $\Pi_{\mathcal H}$ is Euclidean projection onto $\mathcal H$.
\end{proposition}

\begin{proof}
If $m(\theta)\geq0$, then $y(\theta)\in\mathcal H$ and both distances
vanish. Otherwise, any feasible displacement $\delta$ satisfies
$\langle\delta,\hat d\rangle\geq-m(\theta)$.
Cauchy--Schwarz gives $\lVert\delta\rVert_2\geq-m(\theta)$, with equality
attained by $\delta=-m(\theta)\hat d$.
Squaring this minimum distance proves the claim.
\end{proof}

\begin{proposition}[Local margin effect of a clipped correction]
\label{prop:local-margin-effect}
Suppose $m(\theta)<0$, $\omega,\rho\geq0$, and $\nabla m$ is $L$-Lipschitz
on a convex neighborhood containing the two update segments below.
Set $g=g^{\mathrm{base}}$ and
\[
h=\operatorname{clip}_{\rho\lVert g\rVert_2}(\omega\nabla G(\theta)),
\]
where $\operatorname{clip}$ denotes Euclidean norm clipping.
There exists $\lambda\in[0,\omega]$ with
$h=\lambda\nabla G(\theta)$.
For $\eta>0$, define the ordinary gradient-descent steps
$\theta^+=\theta-\eta(g+h)$ and $\theta^0=\theta-\eta g$. Then
\begin{equation}
m(\theta^+)-m(\theta^0)
=2\eta\lambda[-m(\theta)]\lVert\nabla m(\theta)\rVert_2^2
+r_\eta,
\label{eq:app-local-correction}
\end{equation}
with
\begin{equation}
|r_\eta|\leq\frac{L\eta^2}{2}
\left(\lVert g+h\rVert_2^2+\lVert g\rVert_2^2\right).
\label{eq:app-correction-remainder}
\end{equation}
\end{proposition}

\begin{proof}
Euclidean norm clipping rescales its input by a factor in $[0,1]$, giving
the stated $\lambda$. Since $m(\theta)<0$,
$\nabla G(\theta)=2m(\theta)\nabla m(\theta)$.
Lipschitz continuity of the gradient implies
\[
m(\theta-\eta v)=m(\theta)-\eta\langle\nabla m(\theta),v\rangle+R_v,
\qquad |R_v|\leq\tfrac{L\eta^2}{2}\lVert v\rVert_2^2.
\]
Subtract the expansions for $v=g+h$ and $v=g$.
The linear term is $-\eta\langle\nabla m,h\rangle
=2\eta\lambda[-m]\lVert\nabla m\rVert_2^2$.
The triangle inequality bounds $r_\eta=R_{g+h}-R_g$ as stated.
\end{proof}

\paragraph{Scope.}
Proposition~\ref{prop:local-margin-effect} gives a nonnegative first-order
contribution relative to the base step. Proposition~\ref{prop:response-distance}
characterizes response-space geometry.
\end{document}